\documentclass[letterpaper, 10 pt, conference]{ieeeconf}  % Comment this line out if you need a4paper

\IEEEoverridecommandlockouts                              % This command is only needed if 
\usepackage[utf8]{inputenc}
\let\labelindent\relax
\usepackage{amsfonts}       % An extended set of fonts for use in mathematics

\usepackage{amsthm}
\usepackage{mathtools}      % Loads amsmath
\usepackage{amssymb}        % provides AMS fonts and symbols ex. \mathbb{...}
\usepackage{filecontents}
\usepackage{graphicx}       % enhanced support for graphics
\usepackage{marginnote}     % Add notes
\usepackage{marvosym}       % mvat
\usepackage{overpic}        % Add text on figures
\usepackage{tabularx}
\usepackage{subcaption}
\usepackage{cite}
\usepackage{color}
\usepackage{capt-of}
\usepackage[T1]{fontenc}
\usepackage[normalem]{ulem}
\usepackage{epstopdf}
\usepackage{enumitem}
\usepackage[normalem]{ulem}
\usepackage{lipsum}
\usepackage{url}
\usepackage[hidelinks]{hyperref}
\usepackage[dvipsnames]{xcolor}
\usepackage{diagbox}
\usepackage{multirow}
\usepackage{microtype}
\usepackage{xspace}
\usepackage{caption}
\makeatletter
\newif\if@restonecol
\makeatother

\usepackage[linesnumbered,ruled,vlined]{algorithm2e}%[ruled,vlined]{
\usepackage{algpseudocode}
\usepackage{amsmath}

\newif\ifdraft
\draftfalse
\newif\ifarxiv
\arxivtrue
\arxivfalse

\ifdraft
\usepackage[paperheight=11in,paperwidth=9.5in,
			left=1.25in,right=1.25in,
			top=0.75in,bottom=0.75in,
			heightrounded,marginparwidth=1.2in,
			marginparsep=0.05in]{geometry}
\usepackage{xcolor}
\usepackage{xargs} 
\usepackage[textsize=footnotesize]{todonotes}
\newcommandx{\sh}[2][1=]{\todo[linecolor=blue,
			backgroundcolor=blue!10,bordercolor=blue,#1]{Han: #2}}
\newcommandx{\tg}[2][1=]{\todo[linecolor=orange,
			backgroundcolor=orange!10,bordercolor=orange,#1]{Greaten: #2}}
\newcommandx{\jy}[2][1=]{\todo[linecolor=green,
			backgroundcolor=green!10,bordercolor=green,#1]{JJ: #2}}
\else
\newcommand{\sh}[1]{{}}
\newcommand{\tg}[1]{{}}
\newcommand{\jy}[1]{{}}
\fi

\newif\iftwocolumn
\twocolumntrue

\newtheorem{proposition}{Proposition}[section]

\newtheorem{theorem}{Theorem}[section]
\newtheorem*{definition}{Definition}
\newtheorem*{remark}{Remark}

\makeatletter
\def\subsubsection{\@startsection{subsubsection}% name
                                 {3}% level
                                 {\z@ \hspace*{1mm}}% indent (formerly \parindent)
                                 {0ex plus 0.1ex minus 0.1ex}% before skip
                                 {0ex}% after skip
                                 {\normalfont\normalsize\itshape}}% style
\makeatother

\newcommand{\mpp}{\textsc{MRPP}\xspace}

\newcommand{\unpp}{\texttt{UNPP}\xspace}

\newcommand{\mapd}{\textsc{MAPD}\xspace}

\newcommand{\per}{\textsc{PER}\xspace}

\newcommand{\wcs}{\textsc{WCS}\xspace}
\newcommand{\mwcs}{\textsc{MWCS}\xspace}
\newcommand{\lwcs}{\textsc{LWCS}\xspace}
 
\title{\LARGE \bf
Well-Connected Set and Its Application to Multi-Robot Path Planning
}

\author{Teng Guo   \qquad Jingjin Yu% <-this % stops a space
\thanks{G. Teng, and J. Yu are with the Department of 
Computer Science, Rutgers, the State University of New Jersey, Piscataway, NJ, USA. 
Emails: {\tt\small \{teng.guo, jingjin.yu\}@rutgers.edu}.
}
}

\begin{document}

\maketitle
\thispagestyle{empty}
\pagestyle{empty}

%%%%%%%%%%%%%%%%%%%%%%%%%%%%%%%%%%%%%%%%%%%%%%%%%%%%%%%%%%%%%%%%%%%%%%%%%%%%%%%%
\begin{abstract}
Parking lots and autonomous warehouses for accommodating many vehicles/robots adopt designs in which the underlying graphs are \emph{well-connected} to simplify planning and reduce congestion. 
In this study, we formulate and delve into the \emph{largest well-connected set} (LWCS) problem and explore its applications in layout design for multi-robot path planning.
Roughly speaking, a well-connected set over a connected graph is a set of vertices such that there is a path on the graph connecting any pair of vertices in the set without passing through any additional vertices of the set. 
Identifying an LWCS has many potential high-utility applications, e.g., for determining parking garage layout and capacity, as prioritized planning can be shown to be complete when start/goal configurations belong to an LWCS. 
In this work, we establish that computing an LWCS is NP-complete. We further develop optimal and near-optimal LWCS algorithms, with the near-optimal algorithm targeting large maps. 
A complete prioritized planning method is given for planning paths for multiple robots residing on an LWCS.
% 
%Our research offers a rigorous problem formulation for the \lwcs problem and establishes its NP-hardness. 
% 
%We propose an optimal algorithm and a near-optimal algorithm suitable for large-scale  maps. 
% 
%In addition, a complete prioritized planning method has been proposed for planning paths for multiple robots.

% Our experiments show that our algorithms outperform existing approaches in terms of both solution quality and computation time. We believe that our work provides valuable insights into the maximum well-path-connected vertex set problem and opens up new directions for future research.
\end{abstract}

%%%%%%%%%%%%%%%%%%%%%%%%%%%%%%%%%%%%%%%%%%%%%%%%%%%%%%%%%%%%%%%%%%%%%%%%%%%%%%%%
\section{Introduction}

Designing infrastructures that accommodate many mobile entities (e.g., vehicles, robots, and so on) without causing frequent congestion or deadlock is critical for improving system throughput in real-world applications, e.g., in an autonomous warehouse where many robots roam around. 
% 
%This challenge is particularly relevant in domains, including robotics, where optimizing connectivity can greatly enhance the performance of multi-robot systems. 
% 
A good design generally entails good \emph{environment connectivity} in some sense. 
This paper captures the intuition of ``good connectivity'' with the concept of \emph{well-connect set} (WCS), presents a comprehensive study on computing largest well-connected sets (LWCS), and highlights the application of LWCS in multi-robot path planning (MRPP).
% 
%Well-formed infrastructures play a crucial role in enhancing the efficiency of real-world systems, such as parking and storage facilities. 
% 

To illustrate what a WCS is, let's consider a parking garage. It is essential to design it so vehicles can park without blocking each other, and retrieving a parked vehicle doesn't require moving other vehicles. 
Roughly speaking, the parking spots satisfying these requirements form a WCS, and finding an LWCS is instrumental in determining maximum parking capacity while minimizing congestion.
Well-formed infrastructures based on WCSs are encountered in a broad array of real-world scenarios, including fulfillment warehouses, parking structures, storage systems~\cite{wurman2008coordinating,guo2023toward,azadeh2019robotized,caron2000optimal}, and so on. These infrastructures, designed properly, efficiently facilitate the movement of the enclosed entities, ensuring smooth operations and avoiding blockages.

WCS is especially relevant to MRPP, which involves finding collision-free paths for many mobile robots ~\cite{guo2022sub,yan2013survey,sheng2006distributed,Ma2017LifelongMP,ma2019searching,vcap2015complete}.
Here, the challenge lies in finding feasible paths connecting each robot's start and target positions. 
The concept of \wcs becomes crucial, as it ensures that each robot can reach its destination without traversing other robots' positions, thereby guaranteeing a deadlock-free solution for easily realized prioritized planners.

In this paper, we provide a rigorous formulation for the LWCS problem and establish its computational intractability. We then propose two algorithms for tackling this challenging combinatorial optimization problem. The first algorithm is an exact optimal approach that guarantees finding a largest well-connected vertex set, while the second algorithm offers a suboptimal but highly efficient solution for large-scale instances. As an application, LWCS readily provides prioritized MRPP with completeness guarantees.

\begin{figure}[t]
\vspace{1.5mm}
    \centering
  \begin{overpic}               
        [width=1\linewidth]{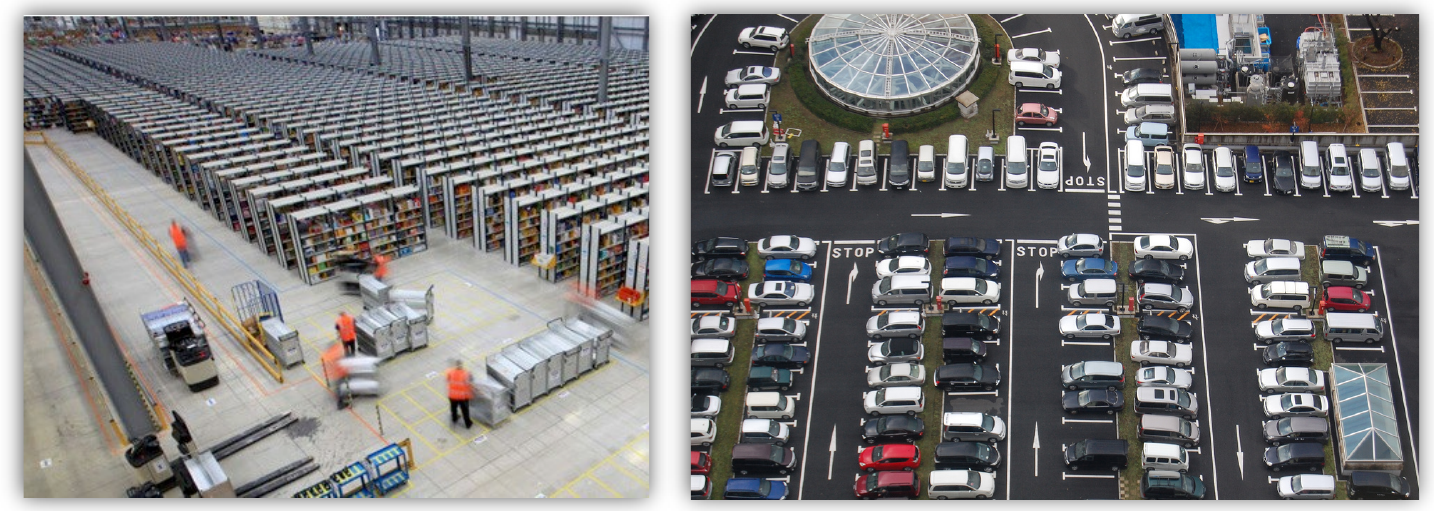}
             \small
             % \put(12.5, 36.5) {(a)}
             % \put(47.5,36.5) {(b)}
             % \put(82.5, 36.5) {(c)}
             \put(20.5, -3) {(a)}
             \put(70.5, -3) {(b)}
        \end{overpic}
\vspace{-3mm}
    \caption{Examples of well-formed infrastructures. (a) Amazon fulfillment warehouse. (b)  A typical parking lot.}
    \label{fig:aplications}
\vspace{-8.5mm}
\end{figure}% 

\textbf{Related work.} The concept of well-connected vertex sets is inspired by well-formed infrastructures \cite{vcap2015complete}. In well-formed infrastructures, such as parking lots and fulfillment warehouses \cite{wurman2008coordinating}, the endpoints are designed to allow multiple robots (vehicles) to move between them without completely obstructing each other, where the endpoint can be a parking slot, a pickup station, or a delivery station. 
Many real-world infrastructures are built in this way to benefit pathfinding and collision avoidance.
Planning collision-free paths that move robots from their start positions to target positions, known as multi-robot path planning or \mpp, is generally NP-hard to optimally solve ~\cite{surynek2010optimization,yu2013structure}.
In real applications, prioritized planning \cite{erdmann1987multiple,silver2005cooperative} is one of the most popular methods used to find collision-free paths for multiple moving robots where the robots are ordered into a sequence and planned one by one such that each robot avoids collisions with the higher-priority robots.
The method performs well in uncluttered settings but is generally incomplete and can fail due to deadlocks in dense environments.
Prior studies \cite{vcap2015complete,ma2019searching} show that prioritized planning with arbitrary ordering is guaranteed to find deadlock-free paths in well-formed environments.
When a problem is not well-formed, it may be possible to find a solution using prioritized planning with a specific priority ordering, as proposed in \cite{ma2019searching}. However, finding such a priority order can be time-consuming or even impossible. 
To our knowledge, no previous studies investigated how to efficiently design a well-formed layout that fully utilizes the workspace.
%We formalize it as a combinatorial search problem on graphs and propose efficient algorithms for solving the problem.
% 
% Alternatively, we propose an algorithm that converts a non-well-formed instance that is unsolvable by prioritized planning, based on \lwcs, to a well-formed instance. Our algorithm runs in polynomial time and is complete when the number of robots is less than half the size of the maximum well-connected vertex set of the map.

\begin{figure}[h]
    \centering
    
  \begin{overpic}               
        [width=1\linewidth]{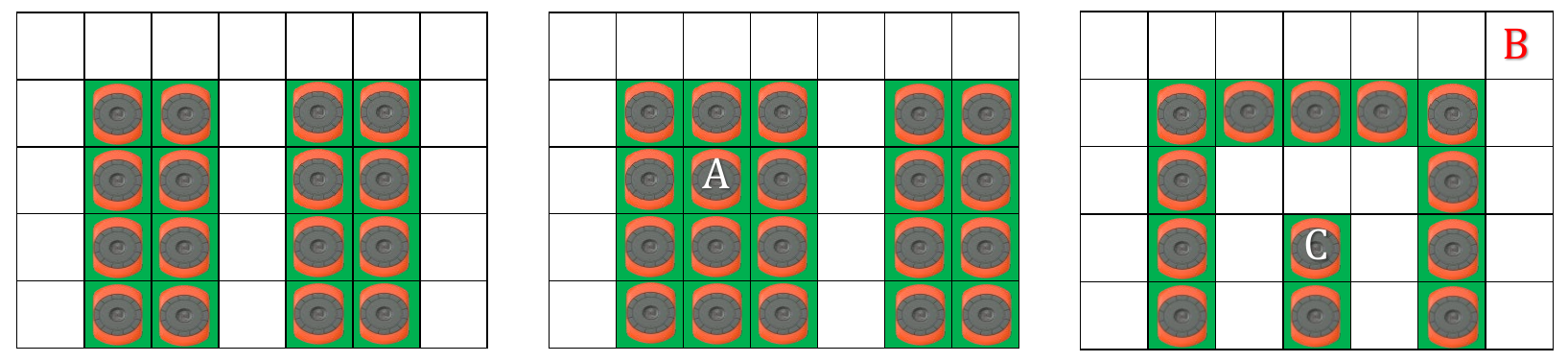}
             \small
             \put(14., -3) {(a)}
             \put(47.9,-3) {(b)}
             \put(82., -3) {(c)}
             % \put(20.5, -3) {(a)}
             % \put(75.5, -3) {(b)}
        \end{overpic}
\vspace{-2mm}
    \caption{(a) The green cells form a WCS. Any robot parked at one of these cells does not block others' move. (b)  An example of a non-WCS. Retrieving one robot, $A$, requires moving some other robots. (c) An example of a SWCS. $B$ has no access to a robot at $C$ without moving others.  }
    \label{fig:well_connected_example}
\vspace{-6mm}
\end{figure}% 

\textbf{Organization.} 
%This paper is organized as follows. 
% 
Section \ref{sec:prelim} provides detailed problem formulations. 
In Section \ref{sec:nphard}, we investigate the theoretical properties and establish the proof of NP-hardness. 
Next, in Section \ref{sec:algorithm}, we present our algorithms for finding \wcs while optimizing the size of the set. 
Section \ref{sec:applications} explores the application of \wcs in multi-robot navigation.
In Section \ref{sec:eval}, we evaluate the effectiveness of our algorithms on various maps. 
Finally, we conclude the paper in Section \ref{sec:conclusion} and discuss future directions for research.
% 
% In \cite{ma2019searching}, some theorems are provided about applying prioritized planning on \mpp.
% % 
% It is said that prioritized planning is incomplete in general cases.
% % 
% There exists a class of problems that are P-solvable which can be solved using a specific priority ordering.
% % 
% And there is also a class of well-formed problems that can be solved using any specific priority ordering.
% % 
% There are several open questions.
% % 
% First, it is unclear if there is an efficient way to know if a problem is P-solvable without searching on the permutation space.
% % 
% Second, since well-defined problems are guaranteed to be solved using prioritized planning, can we find a way to convert a problem into a well-formed problem?
%%%%%%%%%%%%%%%%%%%%%%%%%%%%%%%%%%%%%%%%%%%%%%%%%%%%%%%%%%%%%%%%%%%%%%%%%%%%%%%%%%%

\section{Problem Formulation}\label{sec:prelim}
\subsection{Well-Connected Set}
Let $G(V,E)$ be a connected undirected graph representing the environment with vertex set $V$ and edge set $E$ the edge set. 
A well-connected set (WCS) is defined as follows.
\begin{definition}[Well-Connected Set (WCS)]\label{def:well_connected_set}
On a graph $G(V,E)$, a vertex set $M \subset V$ is well-connected if (i). $\forall u,v\in M, u\neq v$, there exists a path connecting $u,v$ without passing through any $w\in M-\{u,v\}$ , (ii) the induced subgraph of $G$ by the vertex subset $V-M$ is connected.
\end{definition}
WCS enforces a stronger connectivity requirement. If a vertex set $M$ satisfies (i) but violates (ii), we call $M$ a \emph{semi-well-connected set} (SWCS). 
Any WCS is a SWCS set but the opposite is not true (see, e.g.,  Fig.~\ref{fig:well_connected_example}(c)).

For a given $G$, there are many WCSs. We are particularly interested in computing a largest such set.
% 
%This is crucial for increasing the number of available vertices as parking spots or task endpoints for robots or vehicles while avoiding blocking each other. 
% 
Toward that, We introduce two related concepts: the \emph{maximal well-connected set} and the \emph{largest well-connected set}.
% 
%These concepts provide a framework for identifying the largest possible subset of vertices that are well-connected, thus facilitating the efficient use of available resources.
% 

% 
\begin{definition}[Maximal Well-Connected Set (\mwcs)]
A \wcs $M$  is maximal if for any $v\in V-M$, $\{v\}\bigcup M$ is not a well-connected set.
\end{definition}

\begin{definition}[Largest Well-Connected Set (\lwcs)]
 A largest well-connected set $M$  is a WCS with maximum cardinality.    
\end{definition}

By definition, a LWCS is also a MWCS; the opposite is not necessarily true. 
% 
%As mentioned above, finding the maximal/\lwcs,  which is formulated as a combinatorial optimization problem, is helpful in figuring out the parking capacity of a well-formed facility, designing the well-formed layout, and multi-robot coordination.
% 
In this paper, we focus on maximizing the ``capacity" of well-formed infrastructures, or in other words, maximizing the cardinality of the \wcs.
Besides capacity, we also introduce the \emph{path efficiency ratio} (\per) for evaluating how good a layout is from the path-length perspective. 
\begin{definition}[Well-Connected Path (WCP)]
Let $M$ be a WCS. A path $p=(p_0,...,p_k)$ is a well-connected path connecting $p_0$ and $p_k$ if its subpath $(p_1,...,p_{k-1})$ does not pass through any vertex in $M$. 
\end{definition}
If $M$ is a WCS, a WCP connects any two vertices $u,v\in M$. We denote $d_w(u,v)$ as the shortest WCP distance between $u,v$ and $d(u,v)$ as the shortest path distance.
\begin{definition}[Path Efficiency Ratio]
Let $M$ be a WCS of $G$ and $u\in M$ as the reference point (i.e. I/O port), the \emph{path efficiency ratio} w.r.t vertex $u$ is defined as $\frac{\sum_{v\in M}d(u,v)}{\sum_{v\in M}d_w(u,v)}$.
\end{definition}

\section{Theoretic Study}\label{sec:nphard}
In this section, we investigate the property of WCS and prove that finding \lwcs is NP-hard.
% 
% \begin{proposition}\label{p:prop_well_set}
%     If $M$ is a \wcs, then $\forall u\in M$ and $\forall v\in V-M$, there exists a path connecting $u,v$ without passing through any other vertex $w\in M-\{u\}$. In another word, $\forall u\in M$, the subgraph of $G$ induced by $V-(M-\{u\})$ is connected.
% \end{proposition}
% 

A vertex in an undirected connected graph is an \emph{articulation point} (or cut vertex) if removing it (and edges through it) disconnects the graph or increases the number of connected components. 
We observe if a \wcs contains a node $v$, then $v$ should not be an articulation point so as not to violate property (ii) in the definition of \wcs.
\begin{proposition}\label{p:articulation_point}
    If $M$ is a \wcs, for any $M'\subseteq M,v\in M'$, $v$ is not an articulation point of the subgraph induced by $V-M'+\{v\}$.  
\end{proposition}
Next, we investigate the property of the node in \wcs and its neighbors. 
\begin{proposition}\label{p:orphan_neighbor}
Let $M$ be a \wcs. For any $v\in M$, if $|M|>|N(v)|+1$ where $N(v)$ denotes the set of neighbors of $v$, then at least one of its neighbors is not in $M$.
\end{proposition}
\begin{proof}
    Assume $N(v)\subset M$. Because $|M|>|N(v)|+1$, $M-(N(v)\bigcup\{v\})\neq\emptyset$. Let $w\in M-(N(v)\bigcup\{v\})$, then every path from $w$ to $v$ has to pass through a neighboring node of $v$, which contradicts property (i) of \wcs.
\end{proof}
We now prove that finding a \lwcs is NP-hard.
\begin{theorem}[Intractability]
    Finding the \lwcs is NP-hard.
\end{theorem}
\begin{proof}
    We give the proof by using a reduction from  \emph{3SAT} \cite{garey1979computers}. Let $(X,C)$ be an arbitrary instance of 3SAT with $|X|=n$ variables $x_1,...,x_n$ and $|C|=m$ clauses $c_1,...,c_m$, in which $c_j=l_j^1\vee l_j^2\vee l_j^3$.
    Without loss of generality, we may assume that the set of all literals, $l^k_j$'s, contain both unnegated and negated forms of each variable $x_i$.

    From the 3SAT instance, a \lwcs instance is constructed as follows. 
    For each variable $x_i$, we create three nodes, one node is for $x_i$, one node is for its negation $\Bar{x}_i$, and $y_i$ which is a gadget node. 
    The three nodes are connected with each other with edges and form a triangle gadget.
    For each clause variable $c_i=l_j^1\vee l_j^2\vee l_j^3$, we create a node and connect it to the three nodes that are associated with $\Bar{l}_j^1,\Bar{l}_j^2,\Bar{l}_j^3$.
    Finally, we create an auxiliary node $z$ and add an edge between $z$ and each literal node. 
    Fig.~\ref{fig:np_hardness} gives the complete graph constructed from the 3CNF formula $(x_1\vee x_2\vee x_3)\wedge(x_2\vee \Bar{x}_3\vee \Bar{x}_4)
    \wedge (\Bar{x}_1\vee x_3\vee x_4)\wedge(x_1\vee\Bar{x}_2\vee\Bar{x}_4)$.

\begin{figure}[h]
    \centering
  \begin{overpic}               
        [width=0.75\linewidth]{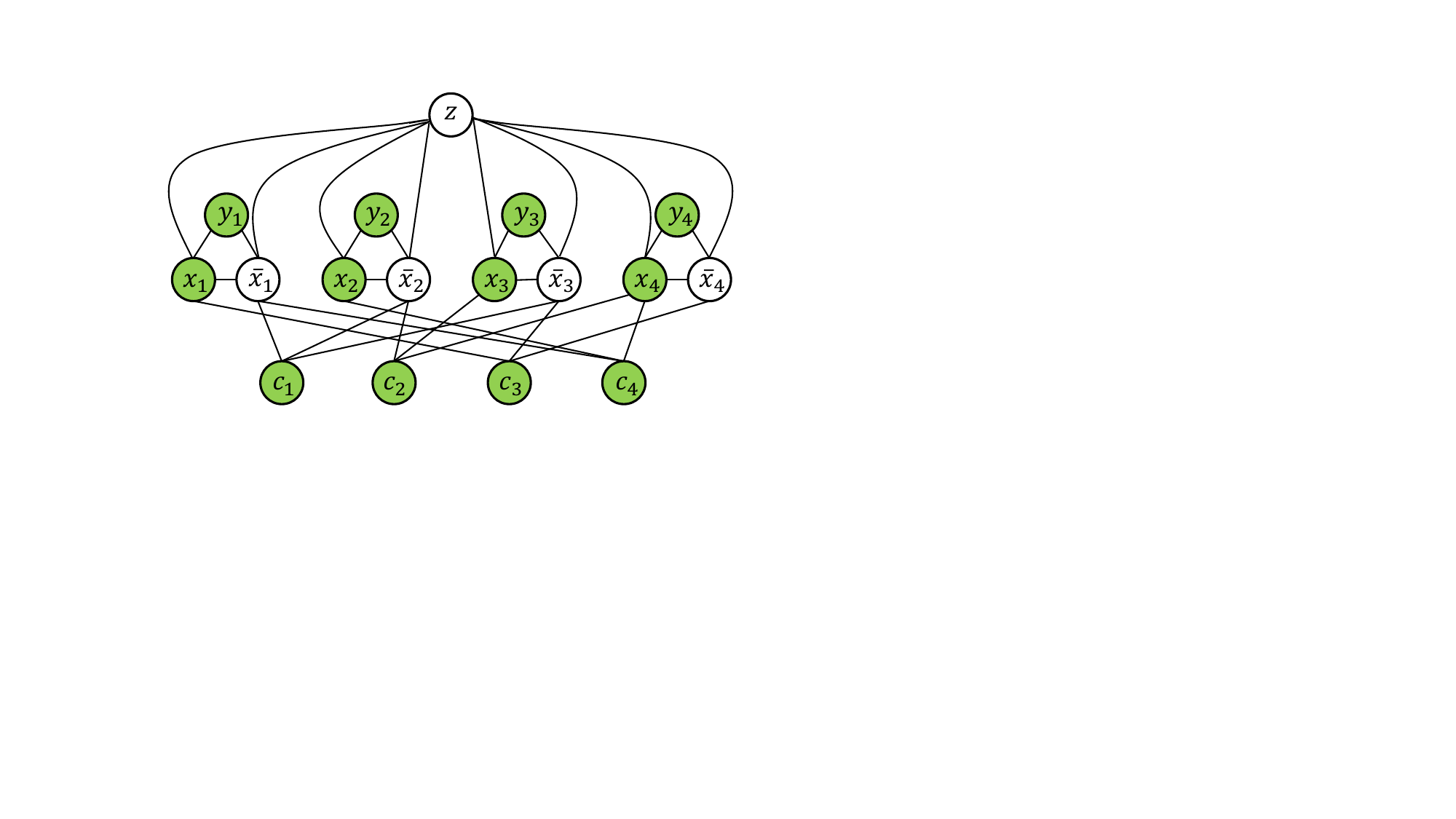}
             \small
             % \put(12.5, 36.5) {(a)}
             % \put(47.5,36.5) {(b)}
             % \put(82.5, 36.5) {(c)}
        \end{overpic}
%\vspace{-3mm}
    \caption{The graph derived from the 3SAT instance $(x_1\vee x_2\vee x_3)\wedge(x_2\vee \Bar{x}_3\vee \Bar{x}_4)
    \wedge (\Bar{x}_1\vee x_3\vee x_4)\wedge(x_1\vee\Bar{x}_2\vee\Bar{x}_4)$. The green vertices form a \lwcs of size 12. By setting the literals of green vertices to true, the 3SAT instance is satisfied.}
    \label{fig:np_hardness}
  \vspace{-7mm}
\end{figure}% 

To find a \lwcs, we observe that for each triangle gadget formed by $(x_i,\Bar{x}_i, y_i)$,  the node $x_i, \Bar{x}_i$ should not be selected at the same time.
    Otherwise, if node $y_i$ is not selected, it would be completely isolated, which violates Proposition \ref{p:articulation_point}. 
    If $y_i$ is also selected, then every path to $y_i$ from other nodes must always pass through either node $x_i$ or node $\Bar{x}_i$, violating \wcs definition.
    This implies that for each triangle, at most two nodes can be selected and added to the set, and one of the nodes must be $y_i$.
    A second observation is if the clause node $c_j=l_j^1\vee l_j^2\vee l_j^3$ is selected as a vertex of the \wcs, the nodes $\{\Bar{l}_j^1,\Bar{l}_j^2,\Bar{l}_j^3\}$ connected to $c_j$ cannot be selected simultaneously.
    Otherwise, node $c_i$ would be blocked by the three nodes. 
    
    If the 3SAT instance is satisfiable, let $\widetilde{X}=\{\widetilde{x}_1,...,\widetilde{x}_n\}$ be an assignment of the truth values to the variables. 
    Based on these observations, the \mwcs of size $2n+m$ is constructed as $\widetilde{X}\bigcup C\bigcup Y$, where $Y=\{y_1,...,y_n\}$.
    On the other hand, if the reduced graph has a \wcs of size $2n+m$, for each triangle gadget, we can only choose two, and one of them must be the gadget node.
    Thus, all the gadget nodes should be selected and this contributes $2n$ nodes.
    For the remaining $m$ nodes, we hope that all nodes in $C$ are selected.
    To not violate the well-connectedness, for each clause node $c_j=l_j^1\vee l_j^2\vee l_j^3$ to be selected, at least one of the node from  $(\Bar{l}_j^1,\Bar{l}_j^2,\Bar{l}_j^3)$ should not be selected. This is equivalent to ensuring that $c_j=l_j^1\vee l_j^2\vee l_j^3$ is true. Therefore, if the node $\widetilde{x}_i$ is selected for each $i$, we set $\widetilde{x}_i=true$, then the resulting assignment $\widetilde{X}$ satisfies the 3SAT instance.
\end{proof}

\ifarxiv
\begin{remark}
Although finding a \lwcs is generally NP-hard, for certain special classes of graphs, a \lwcs can be obtained easily.
For example, in a complete graph, the \lwcs is the set of all nodes in the graph. 
In a tree graph, the \lwcs is the set that contains all the nodes of degree one. 
In a complete bipartite graph $B(U,V)$, the \lwcs is formed by choosing any $|U|-1$ vertices from $U$ and any $|V|-1$ vertices from $V$.
Finding the maximum well-semi-connected set is also NP-hard, which can be proven by reduction from $3SAT$ in a similar manner to the proof for the \lwcs.
\end{remark}
\fi
Next, we investigate the upper bound for the \lwcs denoted as $M^*$.
%not tight 
\begin{proposition}\label{p:upperbound}
    Denote $\Delta$ as the maximum degree of the node of $G$. We have $|M^*|\leq \max(\frac{\Delta-1}{\Delta}|V|,\Delta+1)$.
\end{proposition}
\begin{proof}
For each $v\in M^*$, at least one of its neighbors $u$ should be in $V-M^*$.
Charge $v$ to $u$.
A node $u\in V-M^*$ can be charged at most $\Delta-1$ times.
Hence, we have $|V-M^*|\geq \frac{|M^*|}{\Delta-1}$.
Since every node is either in $M^*$ or $V -M^*$ we have $|M^*| + |V - M^*| = |V|\geq (1+\frac{1}{\Delta-1})|M^*|$.
On the other hand, it is possible that $M^*=N(v)\bigcup\{v\}$ for a node $v$ of degree $\Delta$.
Therefore, $|M^*|\leq \max(\frac{\Delta-1}{\Delta}|V|,\Delta+1)$.
\end{proof}

\section{Algorithms for finding well-connected sets}\label{sec:algorithm}

\subsection{Algorithm for Finding a \mwcs}
We first present an algorithm to find a \mwcs in a graph, which begins by initializing two sets: $M$ and $P$. 
Set $M$ contains vertices currently included in the \wcs, while $P$ contains those available for adding to the set. 
Initially, $M$ is empty, and $P$ contains all the vertices of $G$.
The algorithm selects a vertex from $P$ and adds it to $M$ while maintaining it as a \wcs. 
At each step, we use Tarjan's algorithm~\cite{tarjan1972depth} to compute the set of articulation vertices for the subgraph induced by $V-M$ and remove all the articulation vertices from $P$. 
Assume $v\in M$ and $u\in N(v)$,  $u$ is called an orphan neighbor of $v$ if $u\not\in M$ and $N(v)-u\subset M$.
We iterate over all the neighboring vertices of $M$ to remove all orphan neighbors from $P$, according to Proposition~\ref{p:orphan_neighbor}. 
If $|P|=0$, the algorithm cannot add another vertex to $M$ while keeping it well-connected. 
To ensure that $M$ is indeed maximal, we check if a $v\in M$ exists, such as $M\subset N(v)\bigcup\{v\}$.
If so, we add the remaining neighbor of $v$ to $M$ so that $M$ is maximal.
Finding the set of articulation vertices using Tarjan's algorithm takes $O(|V|+|E|)$.
As a result, the algorithm for finding a \mwcs takes $O(|M|(|V|+|E|))=O(|V|(|V|+|E|))$.
\vspace{-3mm}

\begin{algorithm}
\begin{small}
\DontPrintSemicolon
\SetKwProg{Fn}{Function}{:}{}
\SetKw{Continue}{continue}
  % \KwIn{Starts $X_S$, goals $X_G$, preferred density $\rho^{*}$}
  \Fn{\textsc{MaximalWVS}({G})}{
 \caption{Maximal Well-Connected Set \label{alg:greedy}}
    $M\leftarrow \{\}, P\leftarrow V$\;
    \While{$|P|!=0$}{
        $AP\leftarrow \texttt{Tarjan}(G'(V-M))$\;
        $NB\leftarrow \texttt{FindOrphanNeighbors}(M)$\;
        $P\leftarrow P-(AP\bigcup NB)$\;
        $u\leftarrow \texttt{ChooseOne}(P)$\;
        $M.add(u)$ and $P.pop(u)$\;
    }
    $M\leftarrow \texttt{AdditionalCheck}(M)$\;
    \Return $M$\; 
}
\end{small}
\end{algorithm}
\vspace{-3mm}

% 
% Also not tight
Next, we establish lower bounds for the \mwcs algorithm.
\begin{proposition}\label{p:lowerbound_one_degree}
Denote $W$ as the set of terminal nodes that contains nodes of degree one. Then $W\subseteq M$.    
\end{proposition}
\begin{proof}
    Every terminal node $u$  is neither an orphan neighbor nor an articulation point of the induced subgraph $G'(V-M+\{u\})$. Thus $W\subseteq M$. 
\end{proof}
\begin{proposition}\label{p:lowerbound}\tg{Not very tight, it can be put into the arxiv version}
Denote $L$ as the length of the longest induced path of $G$. Then $|M|\geq \frac{|V|}{L}$.    
\end{proposition}
\begin{proof}
    % Denote $N_o$ as the set of orphan neighbors of $M$ and $AP$ as the set of articulation points of $G'(V-M)$. Clearly $N_o\bigcup AP=V-M$
    Let $u\in V-M$. 
    Then $u$ is either an orphan neighbor or an articulation point of $G'(V-M)$.
    For every $v\in M$, there exists a WCP $p_v$ connecting $u$ and $v$.
    Every vertex in $V-M$ should be passed through by at least one of those WCPs. 
    Otherwise, let $w$ be the vertex that is not passed through by any of the WCPs.
    Clearly, $w$ cannot be an articulation point or an orphan neighbor.
    Then $w$ can be added to $M$, which contradicts that $M$ is maximal.
    Let $p_v'=p_v-v$.
    Then $\bigcup_{v\in M}p'_v=V-M$. Meanwhile, $|p_v|\leq diam(G(V-M))\leq L$, where $diam(G(V-M))$ is the diameter of the induced subgraph by $V-M$.  We have $|V-M|\leq\sum_{v\in M}|p_v'|\leq |M|(L-1)$.
    On the other hand $|V-M|+|M|=|V|$. Hence $|M|\geq |V|/L$.
\end{proof}
\subsection{Exact Search-Based Algorithm}
We now establish a complete search algorithm to find a \lwcs.
\ifarxiv
% \jy{What do you mean by efficient here?}
A naive algorithm for finding the \lwcs starts by initializing a variable $k$ to 1. The algorithm then searches for a  \wcs of size $k$ in the graph. 
If such a set exists, it is considered a candidate for the \lwcs. 
The algorithm then proceeds to search for a \wcs of size $k+1$. 
If such a set is found, it is taken as the new candidate for the \lwcs. 
If no \wcs of size $k$ is found, the algorithm terminates and returns the \wcs of size $k-1$ as the \lwcs. 
This algorithm uses a simple iterative approach and keeps track of the size of the \lwcs found so far. 
By searching for larger \wcs and updating the candidate accordingly, it ensures that it finds the \lwcs.
When searching a well-connected vertex set of size $k$, there are $\binom{|V|}{k}$ possible subsets. 
For each subset $M$ with $|M|=k$, we could run $k$ times of BFS to check if the subgraph induced by $V-M+\{v\}$ is connected for each $v\in M$, in order to examine if it is a \wcs.
The naive algorithm runs in $O(2^{|V|}|V|(|V|+|E|))$.
To improve the efficiency of the algorithm, we propose a search-based algorithm. 

The algorithm is shown in Alg.~\ref{alg:dfs}.
\vspace{-3mm}
\fi
The algorithm employs a depth-first search (DFS) approach to exhaustively explore all possible vertex sets and select the one with the maximum size and the highest path efficiency ratio (PER).
The algorithm starts by defining three empty sets, $M$, $M^*$, and $visited$, 
where $M$ represents the current set of vertices being explored, $M^*$ represents the set with the maximum size and highest \per found so far, and $visited$ stores all previously explored vertex sets to avoid repeated explorations.
Then, the DFS search is initiated by calling the DFS function with the current set $M$, the set with the maximum size and highest PER found so far $M^*$, and the set of visited vertex sets visited as parameters.
The DFS function starts by checking if the current set $M$ has already been explored before. 
If it has, the function returns immediately. 
Otherwise, $M$ is added to the visited set. 
The function then checks if the current set $M$ is larger than the set with the maximum size and highest PER found so far $M^*$. If it is, then we update $M$ as the current best $M^*$. 
If $M$ has the same size as $M^*$, then the function compares their \per values, and the set with the higher \per becomes $M^*$.
% 
% After updating $M^*$, the algorithm checks whether exploring additional vertices is worth it. If the number of remaining candidate vertices plus the size of the current set $M$ is less than the size of $M^*$, the function returns immediately, as there is no point in exploring further. 

Otherwise, the function generates a list of candidate vertices $P$ that can be added to the current set $M$ without violating the \wcs condition. 
The function then loops over the candidate vertices and recursively calls the DFS function with the current set $M$ unioned with the current candidate vertex and the same $M^*$ and visited sets.

The algorithm continues until all possible vertex sets have been explored, or the condition for returning early is met. 
Similar to the algorithm for the \mwcs, we also perform additional checks. 
For each $v\in M$, we check if $N(v)\bigcup\{v\}$ can be larger than the current solution found to ensure that the final vertex set is the maximum one.
The algorithm's time complexity is dependent on the size and density of the graph $G$, as well as the number of candidate vertices generated at each recursive call. 
In the worst-case scenario, the algorithm has a time complexity of $O(2^{|V|}(|V|+|E|))$, where $V$ is the number of vertices in the graph $G$. 
However, the early termination condition in the DFS function helps avoid exploring unnecessary vertex sets and can significantly reduce the algorithm's execution time.
\begin{algorithm}
\begin{small}
\DontPrintSemicolon
\SetKwProg{Fn}{Function}{:}{}
\SetKw{Continue}{continue}
  % \KwIn{Starts $X_S$, goals $X_G$, preferred density $\rho^{*}$}
  \Fn{$\textsc{DfsSearch}(G)$}{
 \vspace{0.5mm}
 \caption{Largest Well-Connected Set\label{alg:dfs}}
 $M,M^{*},visited\leftarrow \emptyset,\emptyset,\emptyset$\;
 \vspace{0.5mm}
 $\texttt{DFS}(G,M,M^{*},visited)$\;
 \vspace{0.5mm}
    $M^{*}\leftarrow \texttt{AdditionalCheck}(M^*)$\;
 \vspace{0.5mm}
 \Return $M^{*}$ \;
}
\Fn{\textsc{DFS}($G,M,M^{*},visited$)}{
 \vspace{0.5mm}
\lIf{$M\in visited$}{\Return}
$visited.add(M)$\;
% \If{$\texttt{IsWellPathConnected}(C')=\text{false}$}{ \Return\;}
 \vspace{0.5mm}
\lIf{$|M|>|M^{*}|$ or ($|M|=|M^{*}|$ and $\texttt{PER}(M)<\texttt{PER}(M^{*})$)}{
    $M^{*}\leftarrow M$
}
 \vspace{0.5mm}
$AP\leftarrow \texttt{Tarjan}(G'(V-M))$\;
 \vspace{0.5mm}
$NB\leftarrow \texttt{FindOrphanNeighbor}(M)$\;
 \vspace{0.5mm}
$P\leftarrow P-(AP\bigcup NB)$\;
 \vspace{0.5mm}
\lIf{$|P|+|M|<|M^{*}|$}{
\Return
}
 \vspace{0.5mm}
\lForEach{$v$ in $P$}{
 \vspace{0.5mm}
    $\texttt{DFS}(G,M\bigcup\{v\},M^{*},visited)$
    }
}

\end{small}
\end{algorithm}
\vspace{-3mm}

\section{Applications in multi-robot navigation}\label{sec:applications}
We demonstrate how \wcs benefits prioritized multi-robot path planning (\mpp) on 
% 4-connected 
graphs. In a legal move, a robot may cross an edge if the edge is not used by another robot during the same move and the target vertex is not occupied by another robot at the end of the move.
% 
%In \mpp, given an undirected graph $G$ and $n$ robots with starts vertices $\mathcal{S}$ and goal vertices $\mathcal{G}$. 
% 
%Here we consider 4-connected graphs and a simplified robot model where in each step, a robot can move to a neighboring vertex in four cardinal directions or wait at its current position. 
% 
The task is to plan paths with legal moves for all robots to reach their respective goals.
%each robot moving it from its start position to its goal position while avoiding collisions with other robots.
% 
The makespan (the time for all robots to reach their goals) and the total arrival time are two common criteria to evaluate the solution quality.
Previous studies \cite{vcap2015complete,ma2019searching} have established the completeness of prioritized planning in well-formed infrastructures. 
Building on the foundation, we provide algorithms with completeness guarantees for non-well-formed environments.
\begin{definition}[Well-Formed \mpp]\label{def:wf_mpp}
An \mpp instance is well-formed if, for any robot $i$, a path connects its start and goal without traversing any other robots' start or goal vertex.    
\end{definition}

\begin{theorem}\label{theorem:wf_guarantee}
Well-formed \mpp is solvable using prioritized planning with any total priority ordering \cite{vcap2015complete,ma2019searching}.
\end{theorem}

When addressing non-well-formed \mpp, we adopt a simple, effective strategy similar to \cite{guo2022sub} to convert start and goal configurations to intermediate well-connected configurations so that the resulting problems are guaranteed to be solvable by prioritized planning with any total priority ordering.
We call the algorithm \unpp - \textbf{un}labeled \textbf{p}rioritized \textbf{p}lanning. 

We compute a \mwcs/\lwcs $M$ offline.   
The first step in the algorithm is to assign the $2n$ vertices, $\mathcal{S}'',\mathcal{G}''$ in $M$ as the intermediate start vertices and goal vertices.
This is done by solving a min-cost matching problem using the Hungarian algorithm \cite{Kuhn1955}. Collision-free paths are easy to plan in the unlabeled setting with optimality guarantees on makespan and total distance \cite{yu2013multi,yu2012distance}.
% 
%Next, we apply an unlabeled multi-robot path planning algorithm that finds a collision-free path for each robot without labels (i.e., the paths are interchangeable), on the assigned well-connected vertices.
% 
%Unlike the labeled version which is intractable, there exist complete polynomial-time algorithms for optimally solving unlabeled \mpp in terms of makespan and total distance \cite{yu2013multi,yu2012distance}. 
% 
The output of this function is a set of collision-free paths for the robots that route them to a well-formed configuration and the intermediate labeled starting and goal positions $\mathcal{S}', \mathcal{G}'$.
The PrioritizedPlanning function is then called on the resulting intermediate starting and goal positions to generate a deadlock-free path for each robot.
Finally, the paths generated by the Unlabeled Multi-Robot Path Planning and Prioritized Planning functions are concatenated to produce a final solution.
% 
% The \unpp algorithm is an efficient way to find a solution for the multi-robot path planning problem, taking into account the \wcs, which ensures that the paths of the robots do not interfere with each other.
\vspace{-4mm}

\begin{algorithm}
\begin{small}
\DontPrintSemicolon
\SetKwProg{Fn}{Function}{:}{}
\SetKw{Continue}{continue}
  \KwIn{Starts $\mathcal{S}$, goals $\mathcal{G}$, \lwcs/\mwcs $M$}
  \Fn{\textsc{UNPP}({$\mathcal{S},\mathcal{G}$})}{
 \caption{ \unpp\label{alg:unpp}}
$\mathcal{S}'',\mathcal{G}''\leftarrow \texttt{Assignment}(\mathcal{S},\mathcal{G},M)$\;
$P_s,\mathcal{S}'\leftarrow \texttt{UnlabeledMRPP}(\mathcal{S},\mathcal{S}'')$\;
$P_g,\mathcal{G}'\leftarrow \texttt{UnlabeledMRPP}(\mathcal{G},\mathcal{G}'')$\;
$P_m\leftarrow \texttt{PrioritizedPlanning}(\mathcal{S}',\mathcal{G}')$\;
$solution\leftarrow \texttt{Concat}(P_s,P_g,P_m)$\;
\Return $solution$\;
}
\end{small}
\end{algorithm}
\vspace{-7mm}

\begin{theorem}
Denote $n_c$ as the size of the \lwcs of graph $G$, for any \mpp instance with number of robots less than $n_c/2$, regardless of the distribution of starts and goals, \unpp is complete with respect to any priority ordering. 
\end{theorem}
\begin{proof}
When the number of robots $n\leq n_c/2$, it is always possible to select such $\mathcal{S}',\mathcal{G}'$ where $\mathcal{S}'\bigcap\mathcal{G}'=\emptyset$.
Since $\mathcal{S}'\bigcup \mathcal{G}'\subseteq M$, by the definition of \wcs, $\mathcal{S}'\bigcup \mathcal{G}'$ is also a \wcs.
Therefore, the resulting \mpp problem which requires routing robots from $\mathcal{S}'$ to $\mathcal{G}'$ is well-formed.
By Theorem.~\ref{theorem:wf_guarantee}, prioritized planning is guaranteed to solve the subproblem using any priority ordering.
\end{proof}
\unpp  runs in polynomial time for \mpp instances with $n\leq n_c/2$.
We can use the max-flow-based algorithm \cite{yu2013multi} to solve the unlabeled \mpp, which takes $O(n|E|D(G))$ where $D(G)$ is the diameter of the graph $G$, if we use \cite{ford1956maximal} (faster max-flow algorithm can also be used here) to solve the max-flow problem.
In the worst case, an unlabeled \mpp requires $n+|V|-1$ makespan to solve \cite{yu2012distance}.
For the prioritized planning applied on well-formed instances, the makespan is upper bounded by $nD(G)$. 
As we use spatiotemporal A* to plan the individual paths while avoiding collisions with higher-priority robots on a time-expanded graph with edges no more than $n|E|D(G)$ and such a solution is guaranteed to exist,
the worst time complexity is $O(n^2|E|D(G))$.
In summary, \unpp yields worst case time complexity of $O(n^2|E|D(G))$ and its makespan is upper bounded by $2(n+|V|-1)+nD(G)$.

For $\frac{n_c}{2}<n<n_c$, arbitrary priority ordering does not guarantee a solution.
For some robots, its intermediate start vertex in $\mathcal{S}'$ has to be the intermediate goal vertex in $\mathcal{G}'$ of another robot when assigned from the \wcs $M$.
The resulting subproblem is not well-formed.
However, it is possible to solve such an instance by breaking it into several sub-problems and using specific priority ordering to solve it.
To do this, we can first establish a dependency graph to determine the priority ordering of robots. 
The dependency graph consists $n$ nodes representing the robots. If $s_i'=g_j'$, we add a directed edge from node $i$ to node $j$, meaning that $j$ should have higher priority than $i$.  If the resulting dependency graph is a DAG, topological sort can be performed on it to get the priority ordering.
When encountering a cycle, as $n<n_c$, we can break the cycle by moving one of the robots in this cycle to a buffer vertex in $C-\{\mathcal{S}'\bigcup\mathcal{G}'\}$ and perform the topological sort on the remaining robots. 

Finally, we briefly illustrate the application of \wcs in multi-robot pickup and delivery (\mapd) \cite{Ma2017LifelongMP}.
In well-formed \mapd,  each robot can rest in a non-task endpoint forever without preventing other robots from going to their task endpoints, i.e. pickup stations.
The layout of the endpoints forms a \wcs of the graph of the environment.
While it is desirable to increase the number of robots and the endpoints as many as possible to maximize space utilization, it is also important to keep a well-connected layout so that the robots will not block each other for better pathfinding. 
Thus, the maximum number of endpoints is equal to $n_c$, and at most $n_c-1$ robots can be used in the \mapd.

\section{Experiments}\label{sec:eval}
In our evaluation, we first test algorithms that compute \wcs for grids and a set of benchmark maps and then perform evaluations on \mpp problems. 
Since the grids and maps used in our experiments are either 4-connected or 8-connected, the solution we find without additional checks will always be larger than the number of neighbors of a vertex $v$ plus one (i.e., $|N(v)|+1$). 
Therefore, we can safely omit the additional checks in our solution.
All experiments are performed on an Intel® CoreTM i7-6900K CPU
at 3.2GHz with 32GB RAM in Ubuntu 18.4 LTS and implemented in C++. 
\subsection{Grid Experiments}
We test the algorithms on $m\times m$ 4/8-connected grids with varying side lengths $m$. 
The result is shown in Table~\ref{tab:gridexp}.
In ``Random" and ``Greedy", we run the \mwcs algorithm 50 times and return the set with the maximum size.
Random randomly chooses a node from the candidates $P$ and adds it to the set.
In Greedy, to select the next candidate to add to the current \wcs, we sort the candidate nodes in ascending order based on their total shortest distance from any node in the current \wcs. We then select the candidate with the smallest total shortest distance as the next node to add to the \wcs.
To evaluate the running time of Random and Greedy, we take the average of the total execution time over the 50 runs of the algorithm.
To evaluate the path efficiency of the maximum(maximal) \wcs found by each algorithm, we treat each node in  $V$ as the reference point and compute the \per for each node. We then take the average of the \per values to obtain a measure of the overall path efficiency of the algorithms.
In DFS, we set a time limit of 600 seconds to search for a solution. If the time limit is reached before a solution is found, we report the best solution found so far. 

\begin{table*}
\vspace{2mm}
\centering
% \scriptsize
% \tiny
\fontsize{6.8}{8}\selectfont
  \begin{tabular}{|c|ccc|ccc|ccc!{\vrule width 2pt}ccc|ccc|ccc|ccc|ccc|}
        \hline
        \multirow{2}{*}{Side Len} & \multicolumn{3}{c|}{Random$_4$} & \multicolumn{3}{c|}{Greedy$_4$} & \multicolumn{3}{c!{\vrule width 2pt}}{DFS$_4$} & \multicolumn{3}{c|}{Random$_8$} & \multicolumn{3}{c|}{Greedy$_8$} & \multicolumn{3}{c|}{DFS$_8$} \\
        \cline{2-19}
         & $|M|$ & PER & Time & $|M|$ & PER & Time & $|M|$ & PER & Time & $|M|$ & PER & Time & $|M|$ & PER & Time & $|M|$ & PER & Time \\
        \hline
        5  & 14 & 0.89 & 0   &11     & 0.89 &0 &\color{red}{14}&0.89 &17.79 &20   &0.97 &0   &20   &0.93 &0    &\color{red}{20}& 0.97& 336.2 \\
        10 & 55 & 0.60 & 0   &52     & 0.69 &0 &\textbf{60}        &0.67 &600   &72   &0.52 &0   &73   &0.96 &0    &\textbf{74}             &0.85 & 600 \\
        \ifarxiv
        15 & 124& 0.62 & 0.02&122    & 0.73 &0.04   &\textbf{138}       &0.71 &600   &161  &0.60 &0.04&160  &0.85 &0.06 &\textbf{162}            &0.85 &  600\\
            \fi
        20 & 220& 0.73 & 0.07&238    & 0.74 &0.2    &\textbf{242}       &0.67 &600   &283  &0.46 &0.1 &280  &0.82 &0.3  &\textbf{285}            &0.85 &  600\\
         \ifarxiv
        25 & 238& 0.53 & 0.2 &372    & 0.64 &0.7    &\textbf{378}       &0.72 &600   &\textbf{447}  &0.46 &0.3 &434  &0.86 &0.9  &445            &0.64 &  600\\
        \fi
        30 & 487& 0.53 & 0.4 &561    & 0.62 &2.0    &\textbf{561}       &0.68 &600   &\textbf{645}  &0.47 &0.7 &622  &0.84 &2.3  &642            &0.54 & 600\\
        \ifarxiv
        35 & 667& 0.35 & 0.8 &765    & 0.58 &4.7    &\textbf{765}       &0.71 &600   &\textbf{875}  &0.62 &1.3 &842  &0.79 &5.3  &872            &0.62 &  600\\
        \fi
        40 & 872& 0.47 & 1.4 &992    & 0.72 &9.8    &\textbf{992}       &0.70 &600   &1137 &0.56 &2.4 &1097 &0.79 &11.0 &\textbf{1139}           &0.49 &  600\\
        \ifarxiv
        45 & 1098&0.55 & 2.4 &1245   & 0.57 &18.9   &\textbf{1245}      &0.69 &600   &1441 &0.60 &4.1 &1384 &0.82 &22.2 &\textbf{1443}           &0.36 &  600\\
        \fi
        50 & 1360&0.33 & 3.9 &1588   & 0.64 &35.9   &\textbf{1588}      &0.66 &600   &1778 &0.51 &6.5 &1705 &0.82 &39.8 &\textbf{1785}           &0.50 &  600\\
        \hline
    \end{tabular}
\caption{Grid Experiment\label{tab:gridexp} on 4/8-connected square grids. The numbers in red color are optimal solutions found by DFS.}
%\vspace{-2mm}
\end{table*}

Though DFS only finds guaranteed optimal solutions on small grids, the final vertex size it returned is usually larger than the other two methods and has better \per.
Greedy finds larger \wcs on 4-connected grids than Random.
Interestingly, on 8-connected grids, the \mwcs found by Greedy is smaller.
\per of Greedy is generally better than Random.

Through linear regression, we found that on grids, the size of the maximum(maximal) \wcs $|M|$ founded by these algorithms is linearly related to the number of vertices $|V|$. 
Specifically, on 4-connected grids, $|M| \sim 0.63|V|$, and on 8-connected grids, $|M| \sim 0.72|V|$. 
This means that in a square parking lot (or other well-formed infrastructures), if it is considered a 4-connected grid, at most about $63\%$ of the space can be used for parking.

\subsection{Benchmark Maps}
We select several maps from \href{https://movingai.com/benchmarks/grids.html}{2D Pathfinding Benchmarks}~\cite{sturtevant2012benchmarks}.
Here, we use Greedy to compute the suboptimal \lwcs as most of the maps are too large to perform DFS search.
For maps that are not connected, the largest connected component is used.  
The result is presented in Table.~\ref{tab:map_exp}.
And some examples are shown in Fig.~\ref{fig:benchmark_maps}.
Our algorithm is efficient on large and complex maps with tens of thousands of vertices.
The computed vertex set size is roughly $50\%$-$60\%$ of $|V|$ for 4-connected graphs, and $60\%$-$70\%$ for 8-connected graphs.

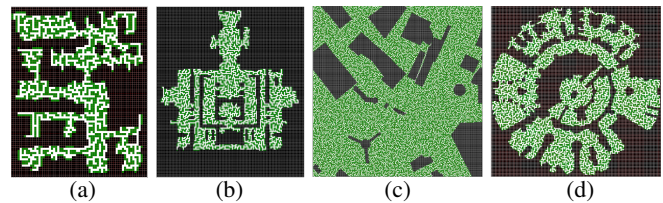
\begin{figure}[!hpbt]
\vspace{-2mm}
    \centering
  \begin{overpic}               
        [width=1\linewidth]{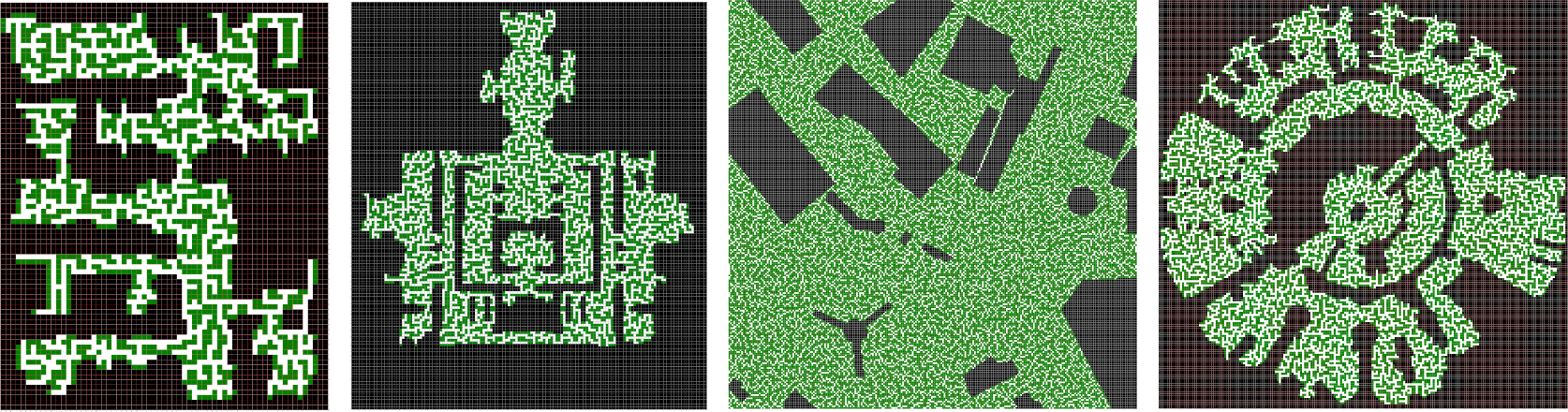}
             \small
             % \put(12.5, 36.5) {(a)}
             % \put(47.5,36.5) {(b)}
             % \put(82.5, 36.5) {(c)}
             \put(9.0, -3) {(a)}
             \put(31.5, -3) {(b)}
             \put(57.5, -3) {(c)}
             \put(85.5, -3) {(d)}
        \end{overpic}
\vspace{-1mm}
    \caption{Examples of the computed \mwcs (colored in green) in different 4-connected grid maps. (a) den312d. (b) ht$\_$chantry. (c) Shanghai$\_$0$\_$256.  (d) lak503d.Zoom in on the digital version to see more details.}
    \vspace{-3mm}
    \label{fig:benchmark_maps}
  \vspace{-5mm}
\end{figure}% 

\begin{table*}
    \centering
    % \scriptsize
    \fontsize{6.9}{8.9}\selectfont
    \begin{tabular}{|l|c|c!{\vrule width 2pt}c|c|c|c!{\vrule width 2pt}c|c|c|c|l|}
        \hline
        \textbf{Map Name}& \textbf{Map grid size} & \;\textbf{\#Vertices} $\mathbf{|V|}$\; & \textbf{\#Edges} $\mathbf{|E_4|}$ & $\mathbf{Time_4 (s)}$ & ${\mathbf{|M_4|}}$&$\mathbf{PER_4}$& \textbf{\#Edges} $\mathbf{|E_8|}$ & $\mathbf{Time_8 (s)}$ & ${\mathbf{|M_8|}}$&$\mathbf{PER_8}$\\ \hline
    
        arena& $49\times 49$ & 2,054 & 3,955 & 2.33  &1,113& 0.68 & 7813 &3.76 &1455 &0.52\\ \hline
        brc202d& $481\times 530$ & 43,151 & 81,512 & 1,685  &22,659 &0.61 &160,277 &2,668 &29,973 & 0.63\\ \hline
        den001d& $80\times 211$ & 8,895 & 16,980 & 20.1 & 4859 &0.77 &33392 & 92.55 & 6233& 0.51 \\ \hline
        den020d& $118\times 89$ & 3102 & 0.12 & 4.48 &1599 &0.89 & 10869 & 9.04 &2104 &0.699\\ \hline
        den312d& $81\times 65$ & 2,445 & 4,391 &3.18   &1,247&0.701 &8,464 & 5.11& 1,663 &0.708\\ \hline
        hrt002d&$50\times 49$ & 754 & 1300 & 0.24 & 377 &0.87&2489&0.38& 510&0.68 \\ \hline
        ht$\_$chantry& $141\times 162$ & 7,461 & 13,963 & 38.87 & 3,889 &0.45 &27222&60.95&5183&0.37 \\ \hline
        lak103d& $49\times 49$ & 859 & 1509 & 0.32  &438&0.84 &2869&0.51&584&0.58\\ \hline
        lak503d& $194\times 194$ & 17,953 & 33,781 & 258.89 &9484 & 0.58& 66,734& 415.60&12,482& 0.48\\ \hline
        lt$\_$warehouse& $130\times 194$ & 5,534 & 10,397 & 18.67 & 2,895& 0.87&20306&31.72&3858&0.63\\ \hline
        NewYork$\_$0$\_$256& $256\times 256$ & 48299 & 94068 & 2000 &26025 &0.40&186935&3469&34054&0.35\\ \hline
        orz201d& $45\times 47$ &745  & 1342 & 0.23 &389 & 0.73& 2604&0.39&513&0.61\\ \hline
        ost003d& $194\times 194$ & 13,214 & 24,999 & 131.82& 7,004 &0.88&49,437&206.30 &9,221 &0.59\\ \hline
        
        random-32-32-20& $32\times 32$ & 819 & 1270 & 0.26 &375 &0.66 &2,487&0.44& 533&0.55 \\ \hline
        Shanghai$\_$0$\_$256& $256\times 256$ & 48,708 & 95,649 & 2,119 & 26,453 &0.32 &190,581 &3,542&34,501 &0.29 \\ \hline
    \end{tabular}
    \caption{Computed suboptimal \lwcs size in selected 4/8-connected maps.      
    \vspace{-5mm}
    \label{tab:map_exp} }
\end{table*}
\subsection{Evaluations of \mpp}
Lastly, we examine the effectiveness of our proposed \mpp method on selected benchmarks. We compare our proposed method with two other prioritized planners, HCA*\cite{silver2005cooperative} and PIBT\cite{Okumura2019PriorityIW}.
For each map, a maximal vertex set is precomputed using the Greedy method. 
To conduct the experiments, we randomly generate 50 instances for each map and the number of robots $n$.
The results are shown in Fig.~\ref{fig:orz201d}-\ref{fig:hrt002d}.
The experimental results demonstrate that our proposed method significantly improves the success rate compared to HCA* and PIBT. Furthermore, although the solution quality is not optimal, it is still reasonably good.
\ifarxiv
It should be emphasized that the proposed method applies to robots that consider kinodynamics. Specifically, in the case of car-like robots, a simple modification can be made by replacing the traditional A* algorithm with the hybrid A* algorithm for single-robot path planning. This adjustment enables the method to cater to car-like motion's unique dynamics and constraints.
\fi
\begin{figure}[!htpb]
    \centering
    \includegraphics[width=1\linewidth]{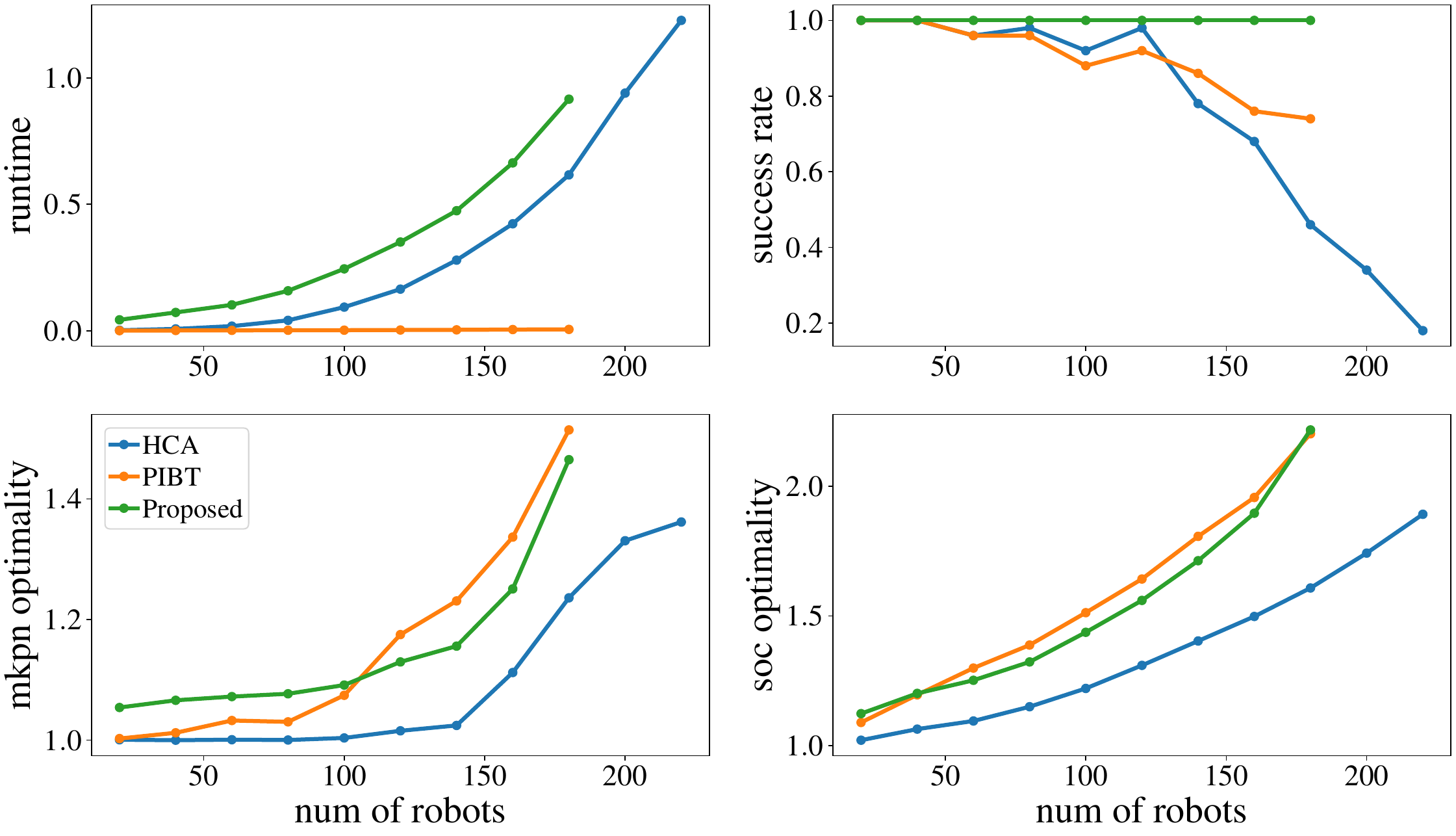}
    \put(-227, 106){\includegraphics[width=0.12\linewidth]{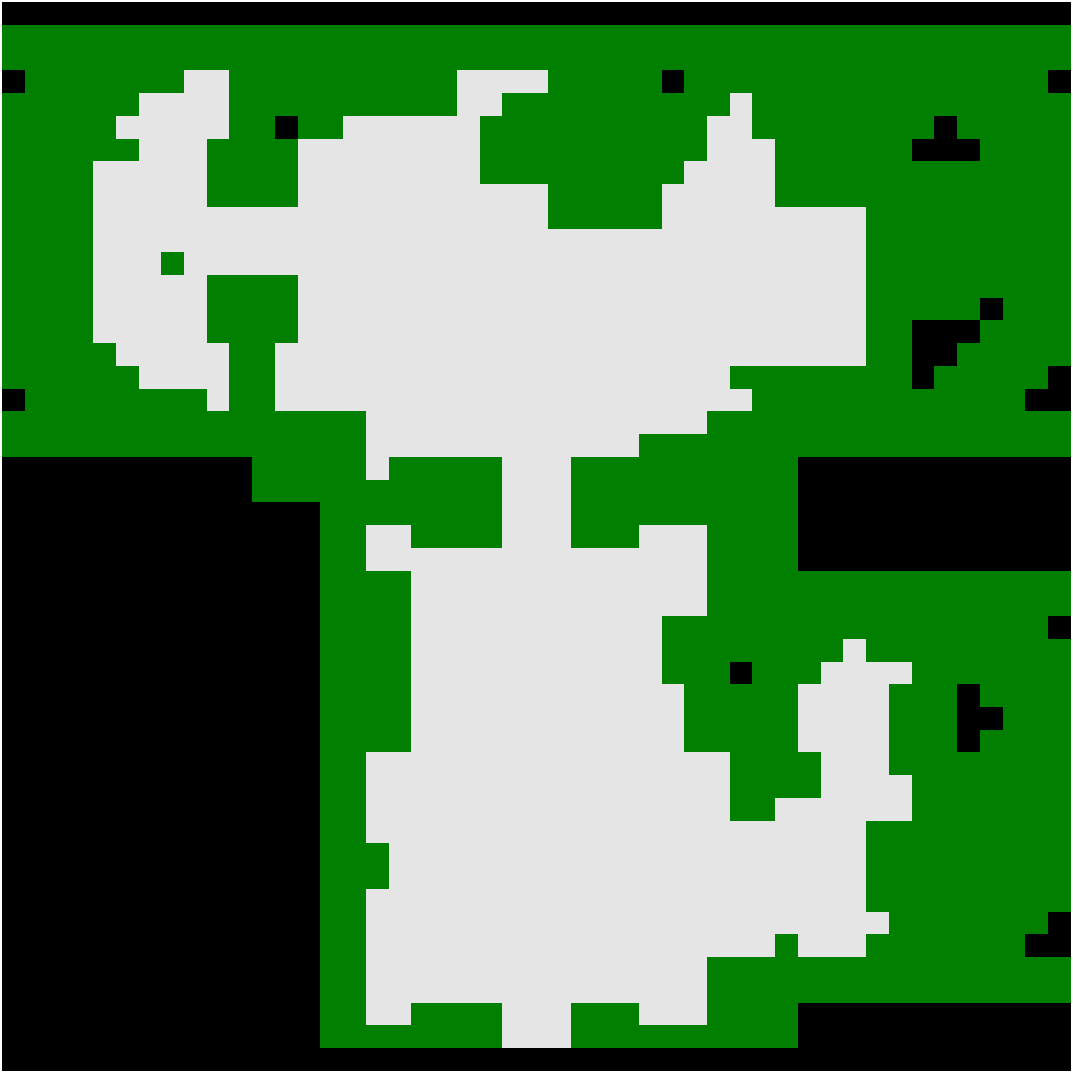}}
    \caption{Experimental results for map orz201d, including computation time, success rate, makespan optimality, and soc optimality, for HCA, PIBT, and the proposed method.}
    \vspace{-4mm}
    \label{fig:orz201d}
\end{figure}

\begin{figure}[!htpb]
    \centering
    \includegraphics[width=1\linewidth]{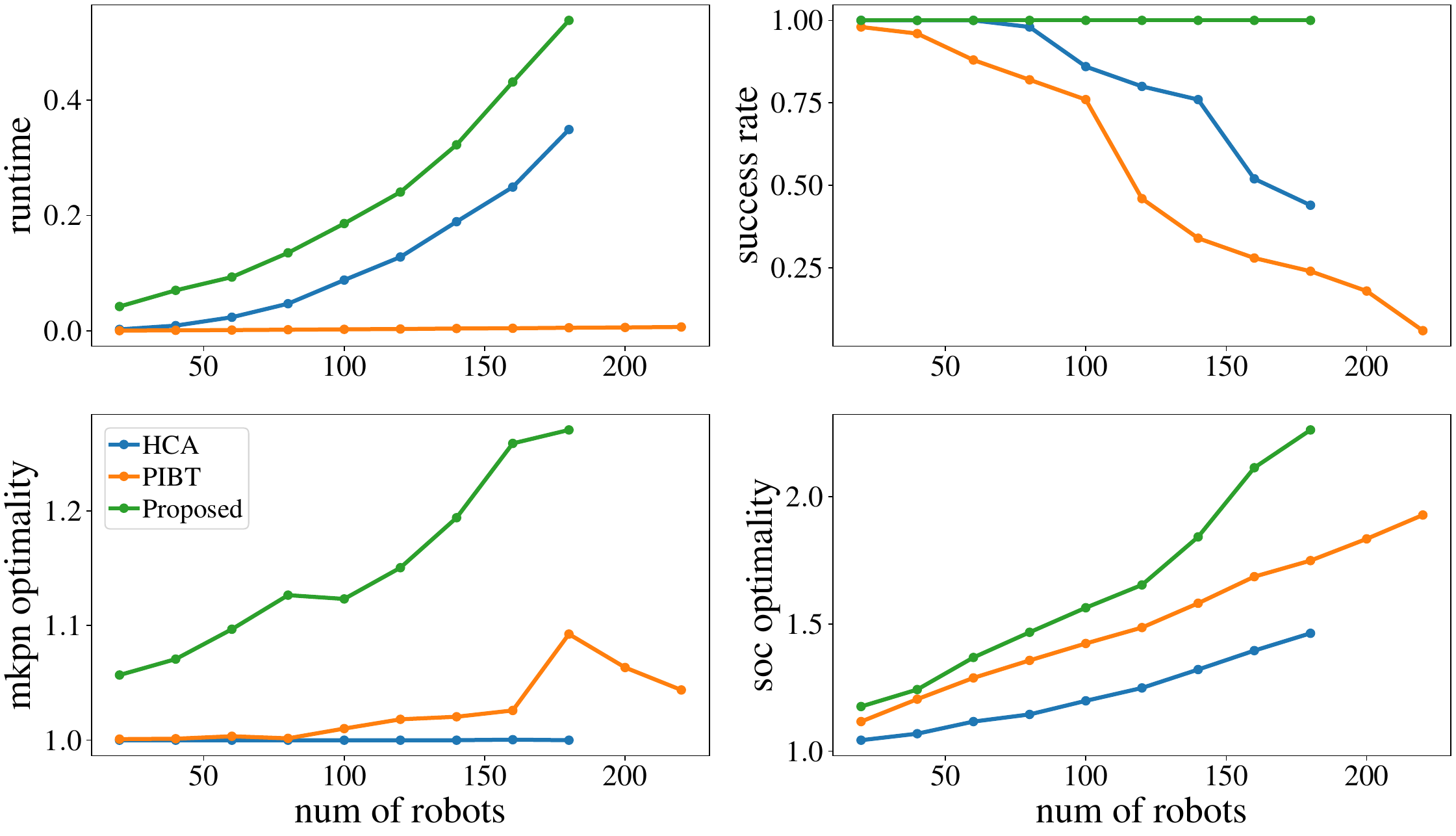}
    \put(-227, 106){\includegraphics[width=0.12\linewidth]{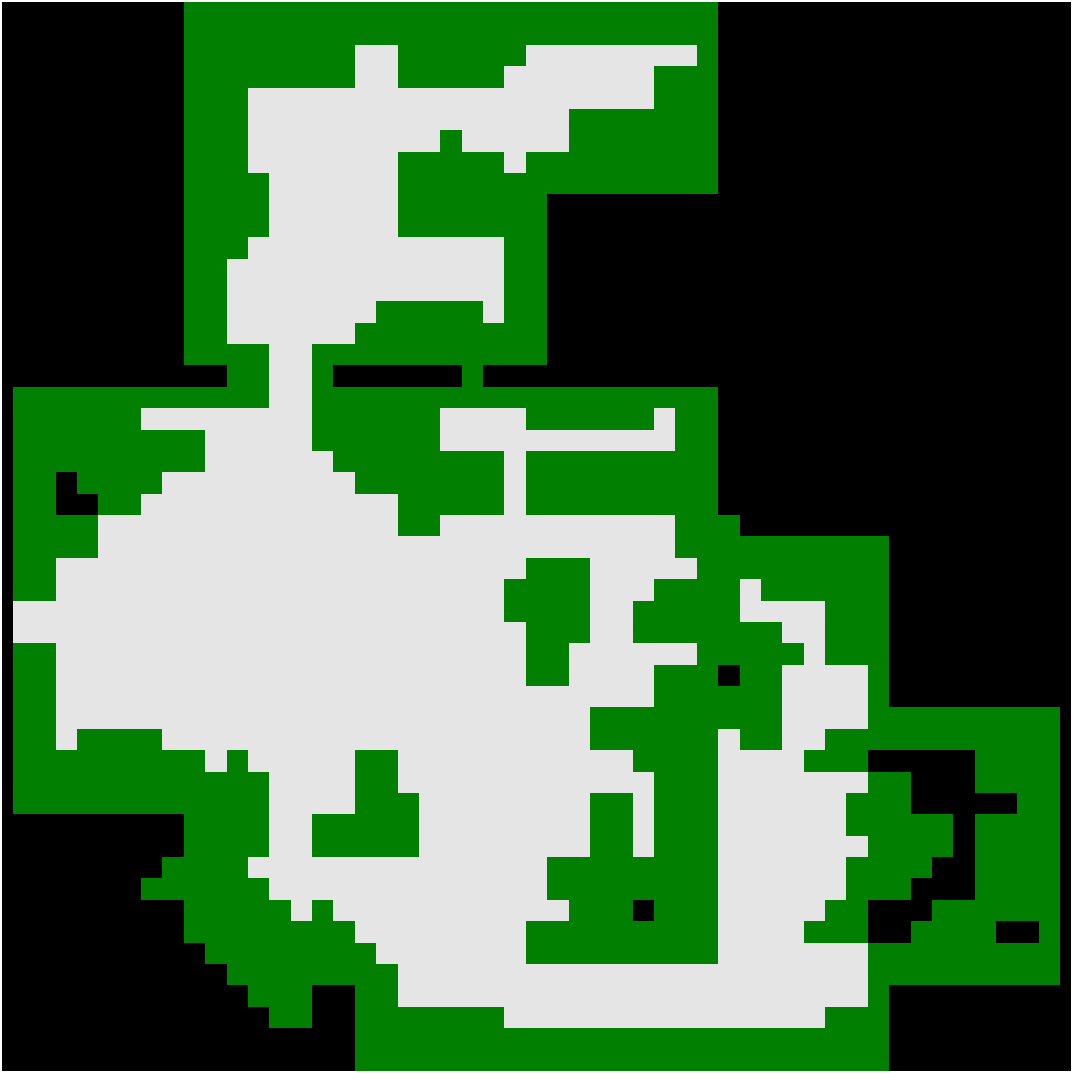}}
    \caption{Experimental results for map hrt002d, including computation time, success rate, makespan optimality, and soc optimality, for HCA, PIBT, and the proposed method.}
    \vspace{-6mm}
    \label{fig:hrt002d}
\end{figure}
\section{Conclusions}\label{sec:conclusion}
In this paper, we have presented a comprehensive study of the \lwcs problem and its applications in multi-robot path planning. 
We provided a rigorous problem formulation and developed two algorithms, an exact optimal algorithm and a suboptimal algorithm, to solve the problem efficiently. 
We have shown that the problem has various real-world applications, such as parking and storage systems, multi-robot coordination, and path planning. 
Our algorithms have been evaluated on various maps to demonstrate their effectiveness in finding solutions. 
Moreover, we have integrated the \lwcs problem with prioritized planning to plan paths for multi-robot systems without encountering deadlocks. 
% 
% Our study contributes to the understanding of the complexity of multi-robot path planning problems and the number of robots and provides a solid foundation for future research in this field. 
Our study enhances comprehension of the relationship between multi-robot path planning complexity, the number of robots, and graph topology, laying a robust groundwork for future research in this domain.
In future work, we plan to investigate the performance of our algorithms in more complex environments and explore their scalability in solving larger instances of the problem.
Additionally, we aim to explore the potential of our algorithms in real-world applications and examine their robustness against uncertainties and disruptions.

%%%%%%%%%%%%%%%%%%%%%%%%%%%%%%%%%%%%%%%%%%%%%%%%%%%%%%%%%%%%%%%%%%%%%%%%%%%%%%%%
% \section*{Appendix}

% \section*{ACKNOWLEDGMENT}

% \bibliographystyle{plainnat}
\bibliographystyle{IEEEtran}
% %\bibliography{references}
\bibliography{all}
%%%%%%%%%%%%%%%%%%%%%%%%%%%%%%%%%%%%%%%%%%%%%%%%%%%%%%%%%%%%%%%%%%%%%%%%%%%%%%%%

\end{document}